\documentclass[letterpaper]{article} 
\usepackage{aaai24}  
\usepackage{times}  
\usepackage{helvet}  
\usepackage{courier}  

\usepackage{amsfonts}
\usepackage[hyphens]{url}  
\usepackage{graphicx} 
\usepackage{natbib,xcolor}  
\usepackage{caption} 
\usepackage{amsmath}
\usepackage{amsthm} 
\theoremstyle{plain}
\newtheorem{theorem}{Theorem}[section]

\newtheorem{lemma}[theorem]{Lemma}

\theoremstyle{definition}
\newtheorem{definition}[theorem]{Definition}

\usepackage{algorithm}
\usepackage{algorithmic}
\usepackage{booktabs}       

\usepackage{newfloat}
\usepackage{listings}
\DeclareCaptionStyle{ruled}{labelfont=normalfont,labelsep=colon,strut=off} 
\floatstyle{ruled}
\newfloat{listing}{tb}{lst}{}
\floatname{listing}{Listing}
\nocopyright

\title{Multi-Dimensional Hyena for Spatial Inductive Bias}
\author {
    Itamar Zimerman,
    Lior Wolf
}
\affiliations {
    Blavatnik School of Computer Science, Tel Aviv University\\
    zimerman1@mail.tau.ac.il, wolf@mail.tau.ac.il
}

\begin{document}
\maketitle

\begin{abstract}
In recent years, Vision Transformers have attracted increasing interest from computer vision researchers. However, the advantage of these transformers over CNNs is only fully manifested when trained over a large dataset, mainly due to the reduced inductive bias towards spatial locality within the transformer's self-attention mechanism. In this work, we present a data-efficient vision transformer that does not rely on self-attention. Instead, it employs a novel generalization to multiple axes of the very recent Hyena layer. We propose several alternative approaches for obtaining this generalization and delve into their unique distinctions and considerations from both empirical and theoretical perspectives.

Our empirical findings indicate that the proposed Hyena N-D layer boosts the performance of various Vision Transformer architectures, such as ViT, Swin, and DeiT across multiple datasets. Furthermore, in the small dataset regime, our Hyena-based ViT is favorable to ViT variants from the recent literature that are specifically designed for solving the same challenge, i.e., working with small datasets or incorporating image-specific inductive bias into the self-attention mechanism. Finally, we show that a hybrid approach that is based on Hyena N-D for the first layers in ViT, followed by layers that incorporate conventional attention, consistently boosts the performance of various vision transformer architectures.
\end{abstract}

\section{Introduction}
Creating a versatile layer designed to effectively process N-dimensional data within deep networks is an important research direction, which has significant implications for key application domains, such as computer vision and speech processing. It is imperative that such a layer not only exhibit strong inductive bias towards N-dimensional data, but also retain the required capacity to exploit extensive datasets. Currently, two primary types of layers dominate N-dimensional data domains: transformers~\cite{vaswani2017attention} and CNNs~\cite{he2016deep,liu2022convnet}. 

Standard CNNs employ relatively small filters~\cite{he2016deep,lecun1989backpropagation,lecun1998gradient}, entailing a high inductive bias, particularly for N-D locality. However, they are less efficient and effective at handling long contexts. Conversely, transformers exhibit a lower inductive bias~\cite{ma2022mega}, but when trained on enough data, they appear to handle N-D data effectively, by processing it as a 1-D sequence with corresponding positional encoding~\cite{dosovitskiy2020image,arnab2021vivit,liu2022video}. 

One advantage transformers hold over CNNs is their ability to deal with varying data lengths and provide a global context at the layer level~\cite{vaswani2017attention}. Yet, their quadratic complexity in sequence length presents obstacles to processing long contexts, which are vital for many tasks. 

This work aims to combine the relative strengths of both CNN and transformers by developing a novel layer that possesses: (i) an inductive bias towards N-dimensional data, (ii) sufficient expressiveness, (iii) a sub-quadratic dependency on sequence length, and (iv) flexibility in processing N-dimensional data of any N-D lengths, while maintaining global context at the layer level. 

As a foundation for this new layer we employ  multi-axes long convolutions, a recent family of layers proven effective for N-dimensional data~\cite{nguyen2022s4nd,baron20232}. These layers employ the convolution of the signal with a multi-axes implicit filter. Unlike prior layers in this field, our implicit filters are not anchored in linear recurrence (similarly to state-space layers~\cite{gu2021combining,s4}). Instead, we extend the very recent Hyena layer~\cite{poli2023hyena} to accommodate N-D data. As our theoretical analysis reveals, this results in a simpler, more expressive, and more efficient model.

Our main contribution is the Hyena N-D layer, which generalizes the recent Hyena layer to multi-dimensional data. We justify our design choices extensively, by empirically and theoretically considering several parametrizations, several decaying structures for incorporating bias of two-dimensional locality, and the first multi-directional variant of Hyena, which has negligible additional computation. Moreover,we are the first to theoretically characterize a form of inductive bias inherent in the family of the Hyena layer. 

As a direct application, we demonstrate that our layer can be used as a drop-in replacement within the ViT backbone to derive a much more data- and memory-efficient model. We also propose a hybrid model that combines attention and Hyena 2-D layers in ViT, further improving performance.

\section{Background and Notations}
\subsubsection{Implicit Global Convolution Layers}
Standard convolution layers are a fundamental building block of deep learning~\citep{cnn1,cnn2,cnn3}. These layers parameterized a convolution filter of size L and C channels with L*C parameters, where each element is defined explicitly. In contrast, an emerging approach implicitly defined the convolution kernel via a learnable function. Namely, the kernel $k_{i}^h$ (filter) at position $i$ and channel $h$ is defined by a function $f^h$ such that $f^h(i) = k_i$. These methods have three main advantages: (i) These layers can operate over an unrestricted context, as opposed to fixed-size explicit filters. (ii) The layers have sub-quadratic time dependency on sequence length, and (iii) As the number of parameters is decoupled from the sequence length, these kernels are regularized by design, which appears to be necessary for their effectiveness~\cite{regu1,regu2}. S4 ~\cite{s4} and state-space layers~\cite{gu2021combining} were the pioneers to show the effectiveness of this approach, by parameterizing convolution kernels via the linear state-space model (SSM), which was then simplified using diagonal and real SSMs~\cite{gupta2022diagonal,gupta2022simplifying}. Similar approaches by~\citet{ma2022mega,lutati2023focus}, use learnable components, including EMA and IIR filters, instead of SSMs to formulate the parameterization. As an alternative, Hyena and CkConv~\cite{romero2021ckconv} established the parameterization by applying standard Feedforward neural network (FFN) layers that operate on positional encoding. These approaches provide superior performance in several areas, such as NLP~\cite{mehta2022long,wang2022pretraining,dao2022hungry}, speech~\cite{saon2023diagonal}, RL~\cite{lu2023structured,david2022decision}, time series analysis, and more, especially in tasks that require capturing long-range dependencies.

\noindent{\bf Implicit N-D global convolution\quad} Recently, this approach extended into multi-dimensional data, using implicit parametrization for N-dimensional filters, which is shown to be an effective method for computer vision tasks~\cite{nguyen2022s4nd,baron20232}. The S4ND~\cite{nguyen2022s4nd} was the first to present the effectiveness of such an approach. It parameterized N-D filters by composing independent SSM-based filters per axis, and during the forward path the filters were aggregated to create a N-D global filter, by taking the outer product of the per-axis filters. This approach was very efficient. However, in~\cite{baron20232} it was shown that the approach of learning kernels separately per axis can be limited in terms of expressiveness, which makes it advisable to leverage it with more expressive mechanisms. In this light, our layer is the first to construct N-D implicit filters without relying on the SSM system.

\noindent{\bf Hyena\quad} The Hyena layer parameterized implicit scalar filters of size $L$ per channel $c \in [C]$ by $ H^c := {h^c}_1 , \cdot , {h^c}_L $ by employing FFN $FFN^c : R^d \rightarrow R$ on positional embedding $pe(l) \in R^d$ such that $\forall l \in [L] : {h^c}_l = FFN^c(\text{PE}(l)) $. This mechanism can generate kernels of any size, enabling the layer to process unrestricted context. Inspired by attention, Hyena implements an expressive data-controlled linear operator, which relies on interleaving implicit long convolutions with element-wise multiplication. In terms of performance, the Hyena layer introduces exceptional performance, reaching transformer quality with a more efficient model, and with sub-quadratic complexity in sequence length.

To add inductive bias towards 1-dimensional locality, the Hyena layer multiplies the filters derived from the FFN with a window function. This function is expressed as follows:
\begin{equation}
\label{eq:hyena1dwindow}
    \text{window}(t) = \text{exp}(- \alpha t) + \gamma 
\end{equation}

\noindent{\bf Vision Transformers\quad} The Vision Transformer (ViT)~\cite{dosovitskiy2020image} is an attention-based model architecture for computer vision tasks. In contrast to conventional Convolutional Neural Networks (CNNs) that utilize local correlations via convolutional filters, the ViT reshapes an image into a 1-D sequence of fixed-size patches, which are processed by a stack of transformer encoder layers. Since transformers are permutation invariant, positional encoding is incorporated. The self-attention mechanism within the transformer enables each patch to be considered in relation to all others, thereby facilitating the learning of both local and long-range dependencies within the image. The output from the transformer is a sequence of embedded patches, with a specific classification token utilized for classification tasks, similar to BERT~\cite{devlin2018bert}. In ViT, the architecture doesn't impose any explicit spatial locality bias, which results in a flexible model that - given a sufficient amount of training data - can capture complex dependencies across the image.

Over the years, the ViT model has seen numerous enhancements. For instance, DeiT~\cite{touvron2021training} optimizes performance through data augmentation, token-based distillation, and regularization, thereby achieving strong benchmark results even with less data. The Swin Transformer~\cite{liu2021swin} introduces a model with more spatial inductive bias, which adapts a hierarchical structure by partitioning images into non-overlapping windows and then processing them hierarchically. Moreover, \cite{dai2021coatnet,guo2022cmt,d2021convit} amplify ViT's efficiency by integrating convolutional layers into the ViT architecture. Furthermore, approaches such as \cite{xie2021segformer,carion2020end,chen2022conditional} introduced specific modifications to ViT, enabling it to excel in detection and semantic segmentation tasks.

\noindent{\bf Notation\quad}
Our notation follows the Hyena literature as closely as possible \cite{poli2023hyena,nguyen2023hyenadna}. 
Specifically, we denote the number of channels by $C$, and the filter on channel $c \in [C]$ by $H^c$. We denote the number of dimensions by $N$, and the sequence length at any dimension by $L_n$ for $n \in [N]$, $L := \Pi_{n=1}^N L_n$ as the total sequence length, $L_{max} := \max_{n=1}^N L_n$ as the maximal sequence length, and $\hat{N}$ as the depth of the Hyena recurrence.

For the notations of the Hyena architecture, we denote the FFN network by $\text{FFN} : \mathbb{R}^M \rightarrow \mathbb{R}^{(\hat{N}+1)C} $, the positional encoding function by $\text(PE) : \mathbb{R} \rightarrow \mathbb{R}^M$, where $M$ is the size of the FFN input layer. Finally, we denote the window function by $\text{window} : \mathbb{R}  \rightarrow \mathbb{R}$.

\section{Method \label{sec:method}}
{\noindent{\bf Motivation }} A well-known drawback of self-attention is its relatively weak inductive bias. This is even more relevant when handling two-dimensional data. In order to design a data-efficient ViT, we choose not to incorporate 2-D inductive bias into the self-attention mechanism in ViT (as done in \cite{liu2021swin,xu2021vitae}), and instead employ an alternative sequence layer within the ViT engine. Recently, several novel sequence layers showed impressive results in 1-D sequence modeling, specifically in improving complexity~\cite{peng2023rwkv,poli2023hyena,dao2022hungry}. This motivated us to explore the utility of such layers as a drop-in replacement for ViT. 

Among those layers, we focus on the Hyena~\cite{poli2023hyena} layer for two main reasons: (i) It is built on simple mechanisms, such as multiplicative element-wise gating and implicit global filters. Hence, it provides a flexible structure that can be modified to incorporate image-specific inductive bias. (ii) Given that traditional convolution layers are known for their significant inductive bias in vision tasks, it is reasonable to assume that the implicit global convolution layers that are part of Hyena would possess similar capabilities. In this light, our work can be considered as a further step towards combining convolution layers and ViT. In contrast to previous work \cite{dai2021coatnet,guo2022cmt,d2021convit}, we employ implicit global convolution layers, rather than standard convolution layers, which do not focus exclusively on short-range dependencies. Furthermore, since the Hyena layer has a sub-quadratic dependency on sequence length, it can lead to a significantly more efficient ViT. This is especially valuable for tasks that involve processing high-resolution images, such as medical imaging.

\medskip
\noindent{\bf The Hyena-ND Layer\quad}
The Hyena layer is composed of three main components: (i) implicit global filters (ii) a data control mechanism in the form of gating (element-wise multiplications), and (iii) a short filter implemented by a 1-D convolutional layer. The scalability of the last two components to accommodate multidimensional data is straightforward, as the second component is dimension-agnostic, and the extension of (iii) to handle multidimensional data can be realized simply, via a standard 2-D convolutional layer. Therefore, our main contribution in introducing the Hyena N-D layer is the creation of N-D implicit filters, which can be utilized for N-D convolution. In the following two sections, we list two alternative strategies for the construction of these filters, and illustrate them in Fig. \ref{fig:mainFigure}. For simplicity, although the Hyena N-D formulation can naturally correspond to $N$ dimensions, we assume $N=2$.

\begin{figure*}[t]
    \centering
    \includegraphics[width=17cm,height=7cm]{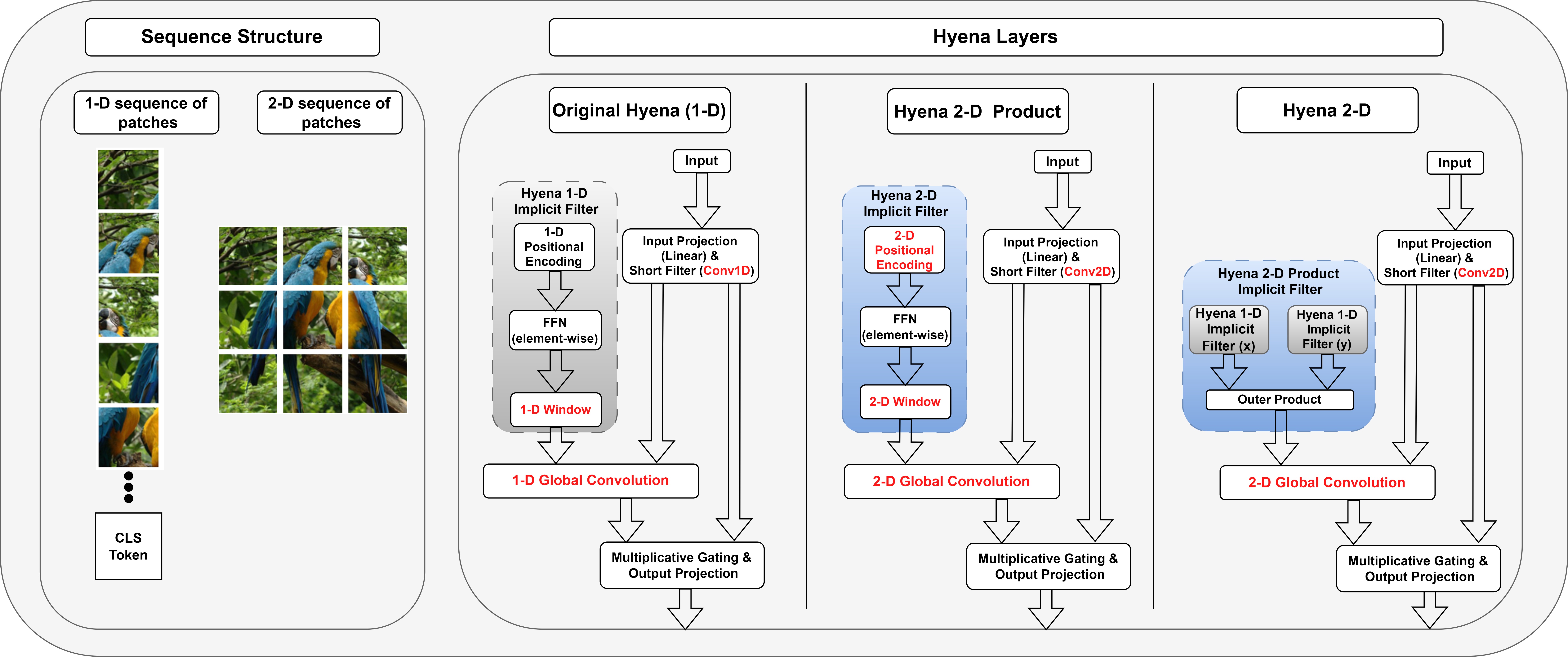}
    \caption{Method: (Left) An image can be organized as a 1-dimensional sequence of patches with a classification token, or as a 2-dimensional sequence of patches. The original attention-based ViT processed the patches as a 1-D sequence; Hyena-based ViT operates directly on a 2-D sequence of patches. (Right) The Hyena 2-D layer architecture is compared to the original Hyena. The modifications are in the parametrization of implicit filters, the type of global convolution, and the short filter type. Hyena 2-D$_{\text{product}}$ uses the implicit filters of Hyena 1-D as a black-box while Hyena 2-D extends the 1-D layer.}\label{fig:mainFigure}
\end{figure*}

\subsubsection{Using N-Dimensional Hyena as a Composition}

The most straightforward way employs multiple independent 1-D filters, similarly to S4ND~\cite{nguyen2022s4nd}. 
To obtain an N-D filter for each channel, a 1-D filter $H^n := (h_{1}^n , h_{2}^n, \hdots, h_{L_n}^n) $ of length $L_n$ is independently learned for each axis $n \in [N]$. The N 1-D filters are then combined to form a single global N-D filter via an outer product operation per channel: 
$$ H = {H^1} \otimes {H^2} \otimes \hdots {H^N}  $$

a notable drawback of this method is that parameterizing each dimension independently is unnatural for several modalities, for example images, and can result in poor inductive bias. We denote the layer as Hyena N-D$_{\text{ product}}$.

\subsubsection{Using Implicit N-Dimensional Hyena Filter}
In contrast to the previous approach for generalizing Hyena to N-dimensional data, this approach attempts to keep the spirit of the original Hyena in the design of N-D implicit filters rather than build a N-dimensional layer on top of the 1-D Hyena layer. To do so, N-D implicit filters and N-D windows are defined.

{\noindent{\bf N-D Implicit Filter }}
The implicit filters of the conventional (i.e., 1-D) Hyena are defined by:
\begin{equation}
    \label{eq:hyenaFilter}
    h_t = \text{window}(t) \cdot \text{FFN}(\text{PE}(t)) 
\end{equation}
where the window function is described in Eq. \ref{eq:hyena1dwindow} and the FFN denotes a simple feed-forward network.
A simple extension of Eq.~\ref{eq:hyenaFilter} into multidimensional filters with $N$ dimensions $n_1, n_2, \cdots , n_N$,  can be described by:
\begin{equation}
   H_{i_1, i_2, \hdots, i_N} = \text{window}(i_1, i_2, \hdots, i_N) \text{FFN}(\text{PE}(i_1, i_2, \hdots, i_N)) 
\end{equation}

\noindent{\bf N-D Window\label{par:Hyena2dWindow} } The Hyena 1-D window is defined by 
\begin{equation}
    \text{window}(t) = \text{exp}(- \alpha t) + \gamma 
\end{equation}
where $t$ is the time-stamp, $\alpha$ is a decaying parameter and $\gamma$ is a bias term. The following two window functions are considered for the 2-D case:
\begin{equation}
    \label{eq:synWindow}
    \text{window}_{\text{symmetric}}(i,j) = \text{exp}(- \alpha (i+j)) + \gamma 
\end{equation}
\begin{equation}
    \label{eq:dimWindow}
    \text{window}_{\text{dimensional}}(i,j) = \text{exp}(- \alpha i+ \beta j) + \gamma 
\end{equation}

In Hyena 1-D, the parameter $\alpha$ changes to regulate effective filter lengths across separate channels. This is accomplished by generating a sequence of evenly spaced non-learnable $\alpha$ values. For the 2-D case, we modified $\alpha$ and $\beta$ across separate channels and shuffled them randomly to obtain a diverse set of windows. We analyzed the two window functions, as well as the decision to make $\alpha, \beta$ and $\gamma$ as constants in Tab \ref{tab:windowVariants}. We also ablate the empirical contribution of the window mechanism by omitting it entirely.


\section{Model Extension}
\paragraph{Multi Directional Layer} 
Since the Hyena layer is causal, previous data elements will not be affected by subsequent ones. While this property is essential for language modeling, it is very unnatural for computer vision tasks, since it limits the model's capability and contradicts the ViT principles. To address this limitation, we introduce a multi-directional extension to our layer. The following two versions are explored: (a) a 4-directional version, in which before each layer, the input is projected into 4 separate representations, then each representation is rotated. For any representation, the Hyena layer is applied, and then a channel-wise linear layer aggregates the 4 signals. (b) a 2-directional version, in which a rotation is applied between the Hyena recurrence steps, or between the Hyena layers for order 2. These strategies are compared empirically in Tab \ref{tab:bidirectAblate}.

\paragraph{Combining with Attention \label{paragraph:HybridMethod}}
Although the empirical analysis in Sec. \ref{sec:results} demonstrates that when dealing with smaller datasets, Hyena 2-D surpasses attention as the core layer of ViT, it is unclear whether attention can be used to boost the model's performance further. Perhaps, Hyena 2D provides distinct benefits that are different from those of the attention model and the two can be fused synergistically to create a better hybrid model. 

To delve deeper into this aspect, we suggest two main strategies for integrating those layers: (i) \textbf{Alternate}: In this approach, for each pair of self-attention layers in the ViT backbone, we replace the first layer with Hyena 2-D. This approach can be interpreted as using Hyena 2-D to add inductive bias to the attention mechanism, similarly to~\cite{ma2022mega,baron20232}. (ii) \textbf{Hyena First:} Employing Hyena 2-D for the first half of the layers, and attention for the rest. The motivation for using Hyena first is that image-specific inductive bias is more important at the lower layers of the architecture, while the top layers integrate information from across the image.

Sec. \ref{subsec:classification} analyzes these methods empirically. To further understand the potential of the Hyena-First approach, we tried a similar version that employs self-attention for the first layers, denoted as \textbf{Attention First}. We found that the Hyena First approach outperformed the others. We also observed that the Alternate approach is superior to attention-free models, which demonstrates that the layers are complementary. This observation could potentially boost performance even in larger models or larger datasets.


\section{Model Analysis}

\subsection{Complexity \label{sec:complexity}}
The Hyena forward path consists of three steps: (i) constructing implicit filters, (ii) Applying N-D convolution, and (iii) computing the input and output projections, where the complexity of the last step is minimal.

Under the assumption that the hidden FFN dimension is smaller than the number of channels ($M \leq C$), the time and space complexity of creating implicit filters in Hyena 1-D, Hyena N-D and Hyena N-D$_{\text{product}}$ are $LCM$. For all Hyena variants, the computation of the kernel does not depend on the batch size $B$, making it more efficient for large batches.

Next, we apply an N-dimensional convolution between the kernel and the input. Since the convolution can be efficiently computed with FFT, the total time complexity
is $O(BCL \log(L)$ for any dimension N, and the total space complexity is $O(BLC)$. As can be seen, the convolution complexity dominates the overall complexity for large batches. An empirical analysis of this advantage in linear space-complexity is given in \ref{subsec:resEfficientArch}.

\subsection{Expressiveness and inductive bias \label{subsec:express}}
We next characterize the expressiveness of the Hyena N-D layer variants, starting by introducing a theoretical analysis of the expressiveness of the Hyena N-D layers and then comparing it to other methods for creating implicit N-D filters.

\noindent{\bf Assumptions\enspace} In this section, every theorem assumes that the FFN network uses sign activations and $M > 1$. Furthermore, for simplicity, both the positional encoding and window functions are considered to be identity functions.

\noindent{\bf Tensor rank as a criterion for expressiveness\enspace}
We start by introducing our criteria for measuring the expressiveness of the Hyena N-D layer. Inspired by~\citet{cohen2016expressive}, which employs tensor rank as a criteria for expressiveness, we apply tensor rank for the N-D kernels constructed in the Hyena N-D layer, and prove the following theorems: 

\begin{theorem}
\label{theorem:kernelsRank1}
 A single channel of the Hyena N-D$_{\textbf{product}}$ implicit filter can only express kernels of rank 1. 
\end{theorem}

\begin{theorem}
\label{theorem:kernelsRankD}
  Given a N-dimensional sequence such that $\forall n \in [N] : L_n = r$, a single channel of the Hyena N-D implicit filter with hidden dimension $F \geq 2Nr$ and at least 2 hidden layers with sign activations can express N-D kernels of tensor rank $r'$ for any $r' \in [2, \hdots ,r]$.
\end{theorem}


These results are based on the unique structure of Hyena filters, which are obtained by employing a learnable function over positional encoding. Hence, we can represent an N-D filter with $N$ dimensions of size $L_n$ per dimension with an equivalent N-dimensional tensor $\mathbb{A}$ such that:
\begin{equation}
\mathbb{A}_{i_1, i_2, \hdots i_N} := \text{MLP} ( \text{PE} (i_1, i_2, \hdots i_N))\,,
\end{equation}
where $ \forall j \in [N] :i_j \in [L_n] $.

Equipped with this formulation, the proof of Theorem \ref{theorem:kernelsRankD} is specified in Appendix \ref{app:proofKernelsFullRank}; the proof of \ref{theorem:kernelsRank1} is trivial, and derives from the fact that to compute a global multi-axis kernel $\mathbb{H}$, Hyena N-D$_{\textbf{product}}$ takes the outer product operation on the per-axis kernels $\mathbb{H}_n \in \mathbb{H}^{L_n\times 1}$ for all $n \in [N] $. Since each kernel is a vector, it is clear that: %
\begin{align}
    \textbf{rank}(\mathbb{H}) = \textbf{rank}(\mathbb{H}_1 \otimes \mathbb{H}_2 \otimes \hdots \otimes \mathbb{H}_D) = 1
\end{align}

\noindent{\bf Inductive bias towards low rank\enspace}
Theorem \ref{theorem:kernelsRankD} was originally designed to evaluate the expressiveness of the Hyena N-D layer. Nevertheless, it offers valuable insights into the implicit regularization and the inductive bias of the implicit filter mechanism. Theorem~\ref{theorem:kernelsRankD} introduces a linear parameter scaling type of regularization and it is evident that when the hidden dimension of the FFN layer increases, the potential rank increases as well, and the filters are biased toward low-rank tensors. 

Since regularization is seen as a crucial attribute for the effectiveness of global convolution layers~\cite{regu1,regu2}, it is imperative to rigorously define the type of regularization present in the Hyena filters. To the best of our knowledge, this is the first time the inductive bias of the Hyena layer has been formalized. 

{\noindent{\bf Comparison of complexity and expressiveness with other layers }} 
Tab.~\ref{tab:Comparison} compares the expressiveness and complexities of implicit N-dimensional convolution layers. The baseline layers are S4ND~\cite{nguyen2022s4nd} and 2D-SSM~\cite{baron20232}. As can be seen, Hyena N-D is the first layer that can express full-rank kernels for any dimension. Moreover, it has the same complexity as the Hyena 1-D layer when given sequences with an equal number of elements for any number of dimensions.
\begin{table}[]
\centering
\small
\begin{tabular}{@{}l@{~}c@{~}cc@{~}c@{}}
\toprule
\hfill Criteria: & \multicolumn{2}{c}{{Complexity}} & \multicolumn{2}{c}{Express} \\
\cmidrule(r){1-1}
\cmidrule(lr){2-3}
\cmidrule(lr){4-5}
Layer & Time & Space & 2-D & N-D \\
\midrule
\midrule
Hyena 1-D  & $LMC$ & $LMC$ & - & -   \\
\midrule
\multicolumn{5}{c}{Multi-dimensional layers}\\
\midrule
S4ND & $ C  (\Tilde{L}+\Tilde{N'})$ & $C (L + N')$ & 1 & 1    \\
2-D SSM & $L L_n C N$ & $L L_n C N$ & 2 & N.a   \\
\midrule
Hy. N-D$_\text{prod}$ & $LMC$ & $LMC$ & 1 & 1    \\
Hy. N-D  & $LMC$ & $LMC$ & 2 & N  \\
\bottomrule
\end{tabular}
\caption{Complexity and expressiveness (Express) for N-dimensional implicit filters with equal length $L_n$ per dimension. Tildes denote log factors. We compare our Hyena N-D variants (Hy.) from Sec. \ref{sec:method}, and the two baselines (i) S4ND \cite{nguyen2022s4nd} and (ii) 2-D SSM \cite{baron20232}. For the baselines, state size is denoted by $N'$. Expressiveness is analyzed using the tensor rank of the kernel. For simplicity, we assume that $C \geq M$.}\label{tab:Comparison}

\end{table}

\section{Experiments\label{sec:results}}
We evaluated our method on image classification benchmarks across several ViT backbones, including ViT, Swin, and DeiT in Sec. \ref{subsec:classification}, followed by empirically justifying the choices made in the Hyena N-D layer design in Sec. \ref{subsec:designChoices}. Finally, in Sec. \ref{subsec:resEfficientArch} we empirically analyzed the memory efficiency of our layer against standard ViT. 

{\noindent{\bf Experimental setup }} All experiments are conducted using PyTorch. The results of all experiments were averaged over 3 seeds, and we set the FFN dimension at 32 for all datasets. As a deliberate decision, we do not perform hyperparameter tuning of the backbone and training procedure, apart from stochastic depth. All hyperparameters are copied from the baseline, which is~\cite{lee2021vision} for ViT and Swin, and the DeiT repository~\cite{touvron2021training} for experiments on CelebA. Naturally, these parameters were optimized for the vanilla (attention-based) transformers.

\subsection{Hyena N-D as the core layer of ViT \label{subsec:classification}}
We evaluate our method on CIFAR-100, Tiny-ImageNet and CelebA, three classification benchmarks with different scales. We report results both for architectures in which the self-attention mechanism is replaced with N-D hyena and for the hybrid methods of Sec.~\ref{paragraph:HybridMethod}. As we show below, there is a clear advantage to the Hyena-first hybrid method over the other variants, which can be considered as ablations.

{\noindent{\bf Baselines }} We compare our models with vanilla attention and two improved versions of attention for the ViT, Swin and DeiT backbone: (i) SL-transformers \citep{lee2021vision}, which constitute a data-efficient version of ViT for handling small datasets, and (ii) 2-D SSM \citep{baron20232} that incorporates inductive bias into the attention mechanism using a layer that is built on top of a two-dimensional state-space model. Both layers are specifically designed to improve the inductive bias of the self-attention layer within the ViT backbone.

{\noindent{\bf ViT experiments}} For the ViT backbone, we first remove the class token, since Hyena N-D operates on an ordered 2-D sequence. Then we replace each attention layer with Hyena, Hyena 2-D, or Hyena 2-D$_{\text{product}}$. As can be seen in the upper part of Tab. \ref{tab:Datasets}, employing Hyena 1-D instead of attention improves the results by 1.44\% on CIFAR-100 and 2.61\% on Tiny-Imagenet. The empirical contribution of using Hyena 2-D instead of Hyena on those two datasets is 0.45\% and 0.32\% respectively. The hybrid models also seem effective. The Hyena-2D First approach consistently surpasses the other approaches, performing on average 1.2\% higher than the Alternate hybrid approach, 3.46 \% higher than the attention first hybrid approach, and 4.23 \% higher than the standard attention model. 

Compared to the recent baselines 2-D SSM and SL-ViT, we found empirically that the Hyena 2-D based ViT is superior to those two variants by a significant margin. For instance, on the ViT backbone, the Hyena 2-D based ViT performs, on average, 3.83\% higher than attention with 2D-SSM and 0.385\% higher than SL-ViT. The results of the hybrid model are even better, but we did not test hybrid models for these variants. 

{\noindent{\bf Swin experiments}} The Swin backbone improves the ViT architecture by adopting two principles: (i) using a hierarchical structure of decreasing size patches across layers, which is implemented in the backbone level, and (ii) using shifted windows for better capture of spatial dependencies, which is implemented efficiently in the layer-level via a modified attention mask. As we replace each attention layer with several Hyena variants that do not support mask handling, we omit the second principle. 

The empirical results, presented in Tab \ref{tab:Datasets}(bottom) show that Hyena-based ViT is favorable to the attention-based models, even without leveraging this shifting strategy. In CIFAR-100, using Hyena 1-D instead of attention improves results by 1.26\%, and in Tiny-Imagenet, by 1.34\%. Using Hyena 2-Dinstead of Hyena 1-D boosts results further, to 1.96\% and 2.03\%, respectively. We observed that the Hyena-based model notably surpasses the baselines. For instance, the performance advantage is 1.645\% over attention with 2D-SSM and 1.17\% above SL-ViT. 

Similarly to ViT, in Swin, the Hyena-2D First hybrid model approach consistently surpasses the other approaches, performing on average 2.42\% higher than the Alternate hybrid approach, 6.18 \% higher than the Attention first hybrid approach, and 5.45 \% higher than the standard Swin model.

\begin{table}[ht]
\centering
\small
\begin{tabular}{@{}l@{~}c@{~}cc@{~}c@{}}
\toprule
\hfill Dataset: & \multicolumn{2}{c}{{CIFAR-100}} & \multicolumn{2}{c}{Tiny-Imagenet} \\
\cmidrule(r){1-1}
\cmidrule(lr){2-3}
\cmidrule(lr){4-5}
Layer & Acc. & \# Params(M) & Acc. & \# Params(M)\\

\midrule
\multicolumn{5}{c}{\fbox{ViT variants:}}\\
ViT   &  72.72 & 2.71 & 55.14 & 2.75    \\
ViT (no CLS)  &  75.27 & 2.71 & 59.34 & 2.75  \\
ViT w. 2-D SSM & 74.07 & 2.72 & 57.66 & 2.75    \\
SL-ViT  &  76.92 &  2.90 & 61.07 &  2.92    \\
 \midrule
Hyena 1-D  &  76.71 & 2.72 & 61.95 & 2.74 \\
Hyena 2-D$_{\text{product}}$   &  \textbf{77.16} & 2.79 & 62.23 & 2.74    \\
Hyena 2-D  &  76.82 & 2.73 & \textbf{62.27} & 2.74    \\
 \midrule
Hybrid Hyena 2-D$_{\text{First}}$ &  \textbf{78.42} & 2.72 & \textbf{64.66} & 2.74   \\
Hybrid Attention$_{\text{First}}$ &  74.97 & 2.72 & 61.2 & 2.74    \\
Hybrid Alternate  &  77.35 & 2.72 & 63.32 & 2.74    \\

\midrule
\midrule
\multicolumn{5}{c}{\fbox{Swin variants:}}\\
Swin   &  77.60 & 7.11 & 60.06 & 7.15    \\
SL-Swin   &  79.99 & 10.2 & 64.95 & 10.4    \\
Swin w/ 2-D SSM  & 80.12 & 7.15 & 65.77 & 7.18    \\
\midrule
Hyena 1-D  &  78.86 & 7.23 & 61.81 & 7.29    \\
Hyena 2-D$_{\text{product}}$   & 80.82 & 7.51 & 65.43 & 7.64    \\
Hyena 2-D &  \textbf{81.31} & 7.28 & \textbf{66.92} & 7.31    \\
\midrule
Hybrid Hyena 2-D$_{\text{First}}$ &  \textbf{81.50} & 7.19 & \textbf{67.06} &  7.23    \\
Hybrid Attention$_{\text{First}}$ &  76.49& 7.19 & 59.70 &  7.23    \\
Hybrid Alternate  &  79.8 & 7.19 & 63.92 & 7.23    \\
\bottomrule
\end{tabular}
\caption{Variants of ViT (top half) and Swin (bottom half) for small datasets}\label{tab:Datasets}
\end{table}

{\noindent{\bf DeiT experiments}}
Similarly to ViT, we first remove the CLS token and measure performance for each layer. We conduct the experiments on the large-scale CelebA dataset. The original image size is 178x218, and it is resized to 224x224 to match the standard DeiT patch size. The dataset includes 40-way multi-label attribute classification. We report the average accuracy for all 40 tasks, training the models for 20 epochs, similarly to the procedure of~\cite{nguyen2022s4nd,baron20232} on this dataset. 

As can be seen in Tab. \ref{tab:DEITcelebA}, contrary to the finding in Tab. \ref{tab:Datasets}, removing the classification token and replacing attention with Hyena 1-D impacts the results negatively. However, when we integrated Hyena 2-D, the results improved by ~6\% over the Hyena 1-D baseline. Incorporating the bi-directional Hyena 2-D variant boosted results by 1.35\%, matching the attention-based model (without a classification token). However, the original DeiT (attention with classification token) is still more accurate.

As before, the Hyena-2D First approach outdoes the other approaches, performing 0.23\% higher than the Alternate hybrid approach, 2.12 \% higher than the Attention first hybrid approach, and 0.66 \% higher than the standard DeiT. It also outperforms by 0.55\% the 2-D SSM-base baseline, which is slightly better than DeiT itself.

\begin{table}[t]
\begin{center}
\begin{tabular}{lcc}
 \toprule
Layer & Acc. & \# Params \\ 
\midrule
DeiT   &  89.73 & 5.532   \\
DeiT (w/o CLS token) &  88.48 & 5.531  \\
DeiT w/ 2-D SSM  &  89.84 & 5.541  \\
\midrule
DeiT w/ Hyena 1-D  &  80.93 & 5.81  \\
DeiT w/ Hyena 2-D$_{\text{product}}$   &  88.16 & 5.84  \\
DeiT w/ Hyena 2-D   &  88.68 & 5.66  \\
\midrule
Hybrid DeiT Hyena 2-D$_{\text{First}}$ &  \textbf{90.39} & 5.61   \\
Hybrid DeiT Attention$_{\text{First}}$ &  88.27 & 5.61  \\
Hybrid DeiT Alternate  &  90.16 & 5.61   \\
\bottomrule
\end{tabular}
\caption{Variants of DeiT for the Celeb-A dataset.}\label{tab:DEITcelebA}
\end{center}
\begin{center}
\small
\begin{tabular}{lccc}
 \toprule
Layer & 1-Dir & 2-Dir &  4-Dir \\ 
\midrule
 Hyena 2-D$_{\text{product}}$  & $86.76$ & $\textbf{88.16}$  & $88.09$  \\
Hyena 2-D & $87.33$ & $\textbf{88.68}$  & $88.31$  \\
\bottomrule
\end{tabular}
\end{center}
\caption{Ablation results for multi-directional methods, on Celeb-A and the DeiT-S backbone.}\label{tab:bidirectAblate}
\small
\begin{center}
\begin{tabular}{lcccc}
 \toprule
Layer & Constant & Learnable &  Symm.&  w/o \\ 
\midrule
 Hyena 2-D$_{\text{product}}$  & $\textbf{86.76}$ & $85.83$  & N/A & $86.54$  \\
Hyena 2-D & $\textbf{87.33}$ & $85.97$  & $86.29$ & $86.79$ \\
\bottomrule
\end{tabular}
\end{center}

\caption{Ablation results for window methods (tested on 1-Dir), on Celeb-A and the DeiT-S backbone.}\label{tab:windowVariants}

\end{table}

\subsection{Model variants\label{subsec:designChoices}}

In this section we justify our design choices.

{\noindent{\bf Multi-directional }} 
Tab. \ref{tab:bidirectAblate} explores two approaches for efficiently modifying the Hyena layer to consider multi-directional data. As expected, for both Hyena N-D and Hyena N-D$_\text{product}$, the multi-directional approach improves the results by $\sim$1-1.5 \%. We observe that the 2-D approach, which rotates the input before each step on the Hyena recurrence performed slightly better than the 4-directional approach. It is also important to note that the 2-directional version has negligible additional computation than the 1-directional variant.

{\noindent{\bf Window function }}
 As detailed in Sec. \ref{par:Hyena2dWindow}, we explore several window functions, which are evaluated in Tab. \ref{tab:windowVariants}. First, we compare the symmetric (Eq.~\ref{eq:synWindow}) and the dimensional (Eq.~\ref{eq:dimWindow}) windows functions within the Hyena 2-D layer. We found that the  dimensional function performs 1.04 \% better, hence we choose it as our standard window function. Next, we ablate the window mechanism by omitting it and observe a degradation in accuracy of 0.54 \% for Hyena 2-D and 0.26 \% for Hyena 2-D$_{\text{product}}$. Finally, we try to learn the window function by parameterizing Eq. \ref{eq:dimWindow} separately for each channel. This decreases the results by 1.36\% for Hyena 2-D and by 0.93\% for Hyena 2-D$_\text{product}$.

\subsection{Efficiency for large amounts of patches \label{subsec:resEfficientArch}}
One additional benefit of Hyena-based ViT compared to Attention-based ViT is its enhanced complexity in terms of time and memory, as detailed in Sec. \ref{sec:complexity}. To evaluate the memory efficiency of Hyena-ViT in comparison to the standard ViT, we conducted experiments using different patch sizes and measured the peak GPU memory consumption during the forward pass. Fig. \ref{fig:memoryConsumption} demonstrates the significantly improved memory consumption of Hyena-ViT.

\begin{figure}[t]
    \centering
    \includegraphics[width=6.5cm,height=5.3cm]{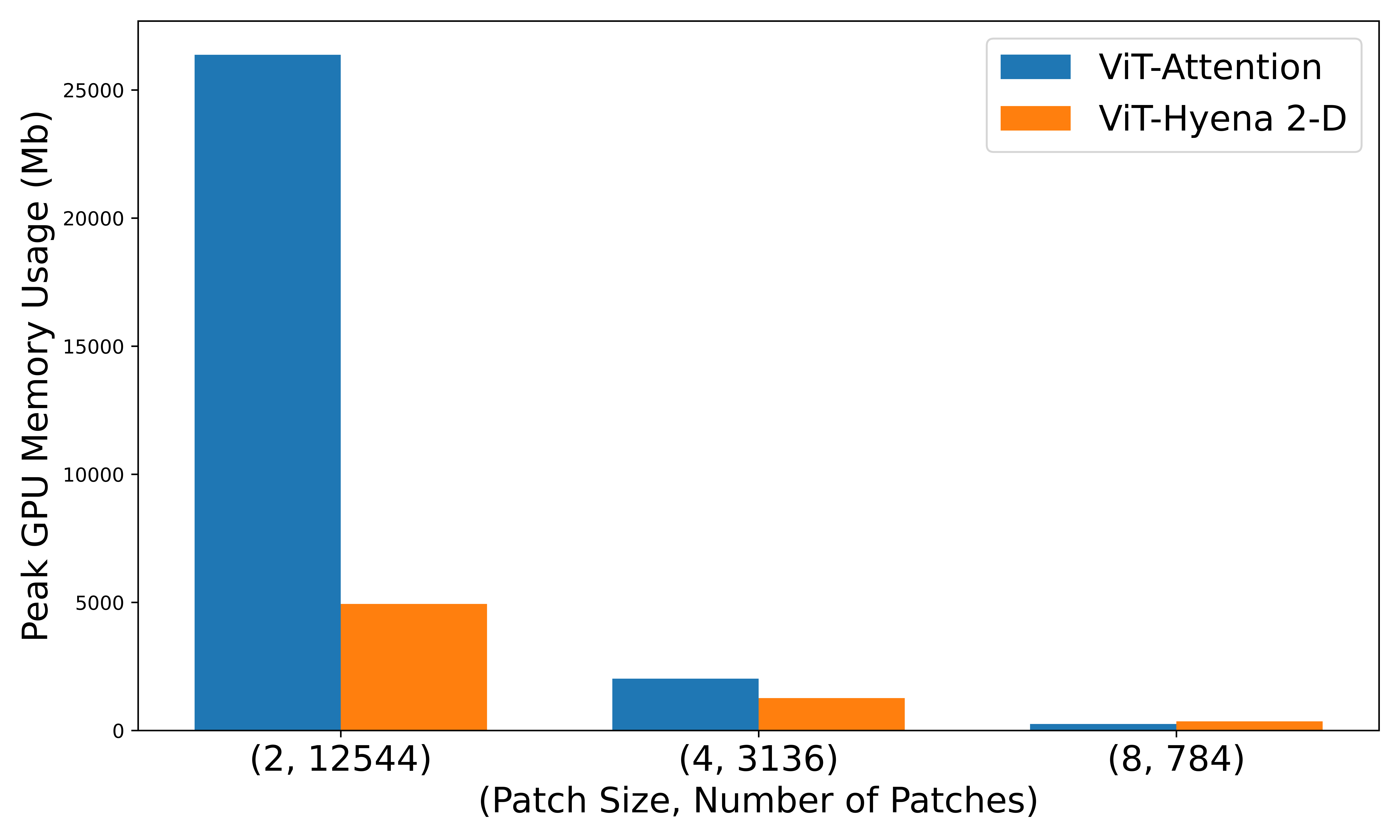}
    \caption{Peak Memory Consumption for Attention-Based and Hyena-Based ViT per patch size}\label{fig:memoryConsumption}
    \centering
\includegraphics[width=7.5cm,height=5.5cm]{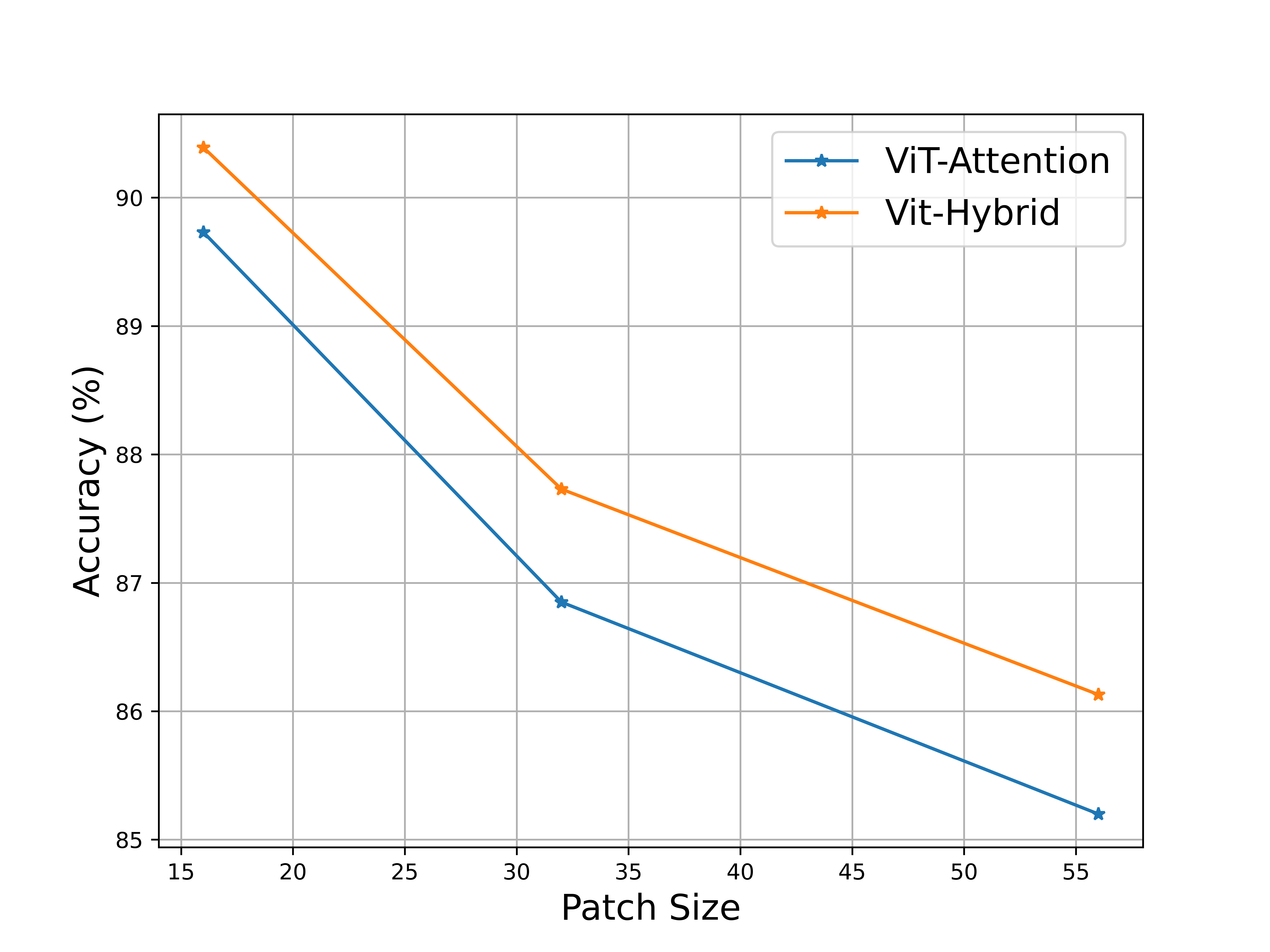}
    \caption{Impact of patch size on the accuracy of Attention-based ViT and Hyena Hybrid ViT. For both Attention and Hyena layers we use the best variants: a CLS token for the Attention layer and a 2-directional layer for the Hyena 2-D.}\label{fig:acc_pr_patchsize}

\end{figure} 

Employing a large number of patches can be critical in two main scenarios: (i) processing high-resolution images, and (ii) working with smaller patches. Previous studies have shown that overly large patches can negatively impact the accuracy of ViT, and smaller patches generally tend to provide better image-specific inductive bias. We examined how the patch size of Hyena-Hybrid ViT affects accuracy in Fig. \ref{fig:acc_pr_patchsize}. The results indicate that Hybrid-ViT also benefits from smaller patches, without a quadratic increase in memory consumption in half of the layers. Thus, Hyena-Hybrid ViT and Hyena-ViT present an opportunity to develop cost-effective ViT models with significantly smaller patches at the same cost.

\section{Limitations}
The move to N-D Hyena-based pooling instead of attention prevents us from using the CLS token, which could be useful. As future work, we would like to add such tokens not as a concatenation, but rather as a conditioning signal. Furthermore, as shown in Swin, self-attention can be easily modified with a domain-dependent mask that enforces a specific shape of inductive bias. Our N-D Hyena lacks such a mechanism. As future work, we would like to investigate whether the N-D window can be modified for similar purposes.

\section{Discussion and Future Work}
In this work, we extend the recent Hyena layer into multi-dimensional data and demonstrate that it can be leveraged to create a data- and memory-efficient variant of ViT. 
We show that a few design choices, such as (i) inserting inductive bias of 2-dimensional locality via employing 2-D instead of 1-D implicit filters, (ii) extending the layer to be a multi-directional operator, and (iii) merging attention and Hyena in a specific manner can notably improve the performance of ViT across various benchmarks.

For future research, we plan on exploring the empirical power of the Hyena 2-D layer in scenarios corresponding to its advantages. As one clear advantage of Hyena-based ViT over vanilla ViT is time- and memory complexity, employing Hyena-ViT in scenarios that require processing large amounts of patches, as well as real-time or low-budget restrictions, is very promising. Finally, we are interested in benchmarking Hyena-ViT on tasks beyond classification, such as segmentation and generation, as well as applying the layer directly to other N-dimensional modalities, such as speech and video.

\newpage
\bibliography{HyenaND}
\newpage
~
\newpage
\appendix
\section{Expressiveness\label{app:proofKernelsFullRank}}
\begin{theorem}
\label{theorem:kernelsRankDP}
  Given an N-dimensional sequence such that $\forall d \in [N] : L_n = r$, a single channel of the Hyena N-D implicit filter with hidden dimension $F \geq 2Nr$ and at least 1 hidden layers with sign activations can express N-D kernels of tensor rank $r'$ for any $r' \in [2, \hdots ,r]$
\end{theorem}
\begin{proof} 
We prove the theorem by showing that a single channel of the Hyena N-D implicit filter can express the N-dimensional identity tensor~\cite{bi2022tensor} for any dimension $D>1$ in Lemma. \ref{lemma:lemmaHyenaID}. Thus, it is clear that a single channel of the Hyena N-D implicit filter can express full-rank kernels, and then we generalize the proof to kernels of any rank $r \in [D]$ 

We start by introducing the identity tensor:
\begin{definition} 
\textbf{Identity tensor~\cite{bi2022tensor}}: The elements of the N-dimensional identity tensor \( \mathbf{I} \) are given by
\[
I_{j_1, j_2 , \cdots, j_N} = 
\begin{cases} 
1 & \text{if } j_1 = j_2 = \cdots = j_N  \\
0 & \text{otherwise}
\end{cases}
\] 
\end{definition} 

\begin{lemma} [Hyena N-D as the identity tensor] \label{lemma:lemmaHyenaID}
The Hyena N-D implicit filter with hidden dimension $F \geq 2Nr$ and at least 2 hidden layers with sign activations can express the identity tensor of dimension $r$.
\label{lemma:idTensor}
\end{lemma}

{\noindent{\bf Generalization to any rank }}

Based on Lemma~\ref{lemma:lemmaHyenaID}, we can construct an FFN that embodies the identity tensor.

A unique characteristic of this FFN construction is the ability to adjust tensor rank through weight modifications in the final layer. Specifically, the weights $\mathbf{W}^3$ are defined as:
$$
{w^3}_{1,i} = 
\begin{cases} 
1 & \text{if } i \leq r'  \\
0 & \text{otherwise}
\end{cases}
$$

By this configuration, only the first r' neurons significantly influence the output, effectively transforming the tensor rank from r to r'

Following the weight adjustment specified above, the tensor can be truncated to its initial $r'$ elements across every dimension. These elements inherently define an identity tensor of rank $r'$ . Any elements beyond this truncated set are zeros. Given the properties of tensor rank, the introduction of these zero elements does not augment the tensor rank.

 \end{proof} 
\begin{proof}[Proof of 
Lemma~\ref{lemma:lemmaHyenaID}]

We prove the lemma using a general example. For simplicity, we assume that the positional encoding function is the identity function. Thus, we consider the following FFN network:

{\noindent{\bf FFN Definition }}
Let the input layer have \( N \) neurons, the first hidden layer \( 2Nr \) neurons, the second hidden layer \( r \) neurons, and the output layer 1 neuron.

Given an input vector \( \mathbf{x} \in [r]^N \) that represents the positional encoding:

The output of the hidden layers is:
\[ \mathbf{h}_1 = \text{sign}\left( \mathbf{W}_1 \mathbf{x} + \mathbf{b}_1 \right),\quad \mathbf{h}_2 = \text{sign}\left( \mathbf{W}_2 \mathbf{h}_1 + \mathbf{b}_2 \right) \]
where \( \mathbf{W}_2 \in \mathbb{R}^{h_2 \times h_1} \), \( \mathbf{b}_2 \in \mathbb{R}^{h_2} \),  \( \mathbf{W}_1 \in \mathbb{R}^{h_1 \times n} \) and \( \mathbf{b}_1 \in \mathbb{R}^{h_1} \).

and the output of the network is:
\[ \mathbf{y} = \text{sign}\left( \mathbf{W}_3 \mathbf{h}_2 + \mathbf{b}_3 \right) \]
where \( \mathbf{W}_3 \in \mathbb{R}^{1 \times h_2} \) and \( \mathbf{b}_3 \in \mathbb{R}^{1} \).

{\noindent{\bf FFN Substitution }}
We will substitute values in $\mathbf{W_1}, \mathbf{W_2}, \mathbf{W_3}$ and $\mathbf{b_1}, \mathbf{b_2}, \mathbf{b_3}$ such that the FFN implements the identity tensor. To achieve this, we use the first layer to obtain a one-hot representation per dimension, which the last two layers will convert into the desired function.

 Thus, we will substitute the values of the first hidden layer $\mathbf{W}_1$ which is denoted by ${w^1}_{i,j}$ and $\mathbf{b}_1$ as follows:
$$
{w^1}_{i,j} = 
\begin{cases} 
1 & \text{if } i \leq Nr \text{ and } \text{floor}(i / N) = j     \\
-1 & \text{if } i > Nr \text{ and }\text{floor}(i / N) = j+Nr \\
0 & \text{otherwise}
\end{cases}
$$

$$
{\mathbf{b}}_{i} = 
\begin{cases} 
i-\frac{1}{2} & \text{if } i \leq Nr\\
i+\frac{1}{2} & \text{if } i > Nr  \\
0 & \text{otherwise}
\end{cases}
$$
Given the output of the first layer $$h_1:= ( {h_1}^1, {h_1}^2, \cdots, {h_1}^{2Nr}) $$ 
it is easy to see that the pair of neurons ${h_1}^{Ni+j},{h_1}^{N(r+i)+j}$ are active if and only if $x_i = j$.

Similarly, we define the second layer as follows:
$$
{w^2}_{i,j} = 
\begin{cases} 
1 & \text{if } j \% r = i  \\
0 & \text{otherwise}
\end{cases}
$$

$$
{\mathbf{b}}_{2} = (-\delta , -\delta  , \cdots , -\delta ), \quad \delta = -2N + \frac{1}{2}
$$

Given the output of the second layer $$h_3:= ( {h_2}^1, {h_2}^2, \cdots, {h_2}^{r}) $$ 
it easy to see that $\forall i \in [N], j \in L_n :{h_2}^{j}=1$ if and only if $ \forall i \in [N] : x_i = j$.

Hence, by using the last layer as an "OR" gate, which can be achieved by setting the last layer as follows:

$$
{w^3}_{1,i} = 
\begin{cases} 
1 & \text{if } i \leq r  \\
0 & \text{otherwise}
\end{cases}
$$

$$
\mathbf{b}_{3} = - \frac{1}{2}
$$
It is clear that the FFN network implements the identity tensor. 


\end{proof}
\end{document}